\documentclass[11pt,english]{article}

\usepackage{hyperref, url}
\usepackage{amsmath,amsfonts,amsthm,amssymb,color,ulem}
\usepackage{latexsym}
\usepackage[tight]{subfigure}
\usepackage{accents}
\usepackage[noend]{algorithm2e}
\usepackage[shortlabels]{enumitem}
\setlist{nolistsep}
\usepackage[margin=1.25in]{geometry}
\usepackage[T1]{fontenc}
\usepackage{babel}

\newcommand{\E}{\mathbb E}
\newcommand{\R}{\mathbb R}
\newcommand{\N}{\mathcal N}

\newcommand{\dP}{\mathbb{P}}
\newcommand{\dQ}{\mathbb{Q}}
\newcommand{\dF}{\mathbb{F}}

\newtheorem{remark}{Remark}
\newtheorem{theorem}{Theorem}

\newtheorem{proposition}{Proposition}

{\begin{enumerate}\item[(#1)]}%
{\end{enumerate}}

{\begin{enumerate}\item[(#1)] [#2 points] }%
{\end{enumerate}}

\title{
Minimax Lower Bounds for Linear Independence Testing
}

\author{%
  Aaditya Ramdas$^*$ \\
  Departments of EECS and Statistics\\
  University of California, Berkeley\\
  \texttt{aramdas@berkeley.edu}
  \and
  David Isenberg\footnote{These authors contributed equally to this work.}\\
  Machine Learning Department\\
  Carnegie Mellon University\\
  \texttt{dki@andrew.cmu.edu}
  \and
    Aarti Singh\\
  Machine Learning Department\\
  Carnegie Mellon University\\
  \texttt{aarti@cs.cmu.edu}
 \and
 Larry Wasserman\\
 Department of Statistics\\
 Carnegie Mellon University\\
 \texttt{larry@stat.cmu.edu}
  }

\begin{document}
\maketitle

\vspace{-5mm}
\begin{abstract}
Linear independence testing is a fundamental information-theoretic and statistical problem that can be posed as follows: given $n$ points $\{(X_i,Y_i)\}^n_{i=1}$ from a $p+q$ dimensional multivariate distribution where $X_i \in \mathbb{R}^p$ and $Y_i \in\mathbb{R}^q$, determine whether $a^T X$ and $b^T Y$ are uncorrelated for every $a \in \mathbb{R}^p, b\in \mathbb{R}^q$ or not. We give  minimax lower bound for this problem (when $p+q,n \to \infty$, $(p+q)/n \leq \kappa < \infty$, without sparsity assumptions). In summary, our results  imply that  $n$ must be at least as large as $\sqrt {pq}/\|\Sigma_{XY}\|_F^2$ for any procedure (test) to have non-trivial power, where $\Sigma_{XY}$ is the cross-covariance matrix of $X,Y$.
We also provide some evidence that the lower bound is tight, by connections to two-sample testing and regression in specific settings. 
\end{abstract}


\section{Introduction}

Linear independence testing is a fundamental problem in information theory and statistical decision theory. One formulation of the problem is as follows: 
\begin{quote}
Given $n$ points $\{(X_i,Y_i)\}^n_{i=1}$ drawn i.i.d. from a $p+q$ dimensional multivariate distribution where $X_i \in \R^p$ and $Y_i \in\R^q$, determine whether $a^T X$ and $b^T Y$ are uncorrelated for all $a \in \R^p, b\in \R^q$ or not.
\end{quote}
Note that for Gaussian distributions, linear independence testing is equivalent to independence testing, and hence in that special case, the problem is equivalent to
\begin{quote}
Given $n$ i.i.d. points $\{(X_i,Y_i)\}^n_{i=1}$ from a $p+q$ dimensional multivariate Gaussian where $X_i \in \R^p$ and $Y_i \in\R^q$, determine whether $X \perp Y$ or not.
\end{quote}

\noindent
Some scientific applications include\\

\begin{itemize}
\item (Medicine) Determine if the injection of a particular vaccine (say MMR) and incidence of a disease (say autism) are independent.\\

\item (Neuroscience) Determine if some feature of a particular stimulus (say color) is independent of the activity in a particular part of the brain (say auditory cortex).\\
\end{itemize}

Naturally, both the above examples and many others are also related to \textit{conditional} independence testing, learning graphical models and inferring causality (does smoking cause cancer? is crime independent of educational status given economic status?). However, before moving onto more complicated problem settings, it is of interest to properly understand the hardness of the aforementioned fundamental problem.
That is precisely the contribution of this paper. We establish minimax lower bounds for linear independence testing, in the high dimensional regime (but without any sparsity assumptions). Our main result can be (morally) summarized as follows:
\begin{quote}
$n$ must be at least as large as $\sqrt{pq}/\|\Sigma_{XY}\|_F^2$ for any procedure (test) to have non-trivial statistical power as $p,q,n \to \infty$ with $(p+q)/n \leq \kappa < \infty$.
\end{quote}

\paragraph{Paper Outline.} In Section \ref{sec:back}, we formally define the problem. In Section \ref{sec:thm}, we state and prove our main theorem. In Section \ref{sec:upper}, we provide some evidence of the tightness of these bounds. We then conclude with some open problems.

\paragraph{Notation}
Let $\N_d(\mu,\Sigma)$ refer to a $d$-variate Gaussian distribution with mean $\mu \in \R^d$ and $d \times d $ positive definite covariance matrix $\Sigma$. We shall also use $\N_d(z; \mu,\Sigma)$ to denote the corresponding Gaussian pdf  at a point $z$ which is given by $(2\pi)^{-d/2} \mathrm{det}(\Sigma)^{-1/2} \exp(-\tfrac1{2} (z-\mu)^T \Sigma (z-\mu))$. $\| \cdot \|$ refers to the standard Euclidean 2-norm. $\R$ denotes the reals, $\E$ denotes expectation (usually with a subscript reflecting the data distribution).


\section{Problem Definition }\label{sec:back}


The problem of linear independence testing can be defined as follows. Given $n$ i.i.d. points $(X_1,Y_1),...,(X_n,Y_n) \in \R^{p+q}$ drawn from a joint 
distribution $\dF$
with covariance matrix $\Sigma$, we wish to ascertain whether $X$ is independent of $Y$ or not. Formally, denoting $\dP,\dQ$ as the marginal distributions of $X \in \R^p$ and $Y \in \R^q$ respectively, this can be formulated as testing if the joint distribution is the product of the marginals or not:
$$
H_0 : \dF = \dP \times \dQ \text{ vs. } H_1 : \dF \neq \dP \times \dQ
$$
Denote $\Sigma_X \in \R^{p \times p}, \Sigma_Y \in \R^{q\times q}$ as the (unknown) population covariance matrices of $X,Y$ respectively (which we assume are full rank and hence positive definite and invertible), and $\Sigma_{XY} \in \R^{p\times q}$ as the cross-covariance matrix of $X,Y$. Then $\Sigma = \left[\begin{matrix} \Sigma_X & \Sigma_{XY} \\ \Sigma_{XY}^T & \Sigma_Y \end{matrix} \right]$ and $\dP \times \dQ$ has the distribution $\N_{p+q} \left(0, \left[\begin{matrix} \Sigma_X & 0 \\ 0 & \Sigma_Y \end{matrix}\right] \right)$. When one is dealing with Gaussian distributions, the problem actually boils down to testing
  if the cross-covariance matrix is zero or not:
$$
H_0 : \Sigma_{XY} = 0 \text{ vs. } H_1 : \Sigma_{XY} \neq 0
$$
\begin{remark}
In some communities, $H_1$ would be stated as $\|\Sigma_{XY}\|_F > \epsilon_n$. The critical value of $\epsilon_n$ that makes the problem hard or easy becomes clear in the main theorem statement.
\end{remark}

\begin{remark}
When $X,Y$ are univariate, one often looks at the correlation coefficient between them. Here too, one can look at an appropriate correlation matrix $R_{XY} = \Sigma_X^{-1/2} \Sigma_{XY} \Sigma_Y^{-1/2}$ and test whether that is zero or not (notice that $\Sigma_{XY} = 0$ iff $R_{XY}=0$, assuming population covariance matrices are invertible). In this paper, we will deal with $\Sigma_{XY}$, but one can imagine doing similar calculations, with morally similar conclusions, for $R_{XY}$ also.
\end{remark}

 A test  $\eta $ is a function from $(X_1,Y_1),...,(X_n,Y_n)$ to $\{0,1\}$, where we reject $H_0$ whenever $\eta=1$. A test is judged by two metrics, its false positive rate $\E_{H_0} \eta$ and its power $\E_{H_1} \eta$. Naturally, that we would like to minimize the former and maximize the latter (we would like $\eta$ to return 0 whenever $H_0$ is true, and return 1 whenever $H_1$ is true).
 In the Neyman-Pearson  paradigm, 
 we only consider tests  that have a type-I error of at most a prespecified level $\alpha$. Let us call the set of all such tests as
\begin{equation}\label{eq:tests}
[\eta]_{n,p,q,\alpha} := \{\eta : \R^{n\times (p+q)} \to \{0,1\}, \E_{H_0} \eta \leq \alpha  \}.
\end{equation}
A test is judged by its power 
$\E_{H_1} \eta$ (one minus type-II error), 
and we say that a test $\eta \in [\eta]_{n,p,q,\alpha}$ is consistent in the high-dimensional setting when its power 
$\E_{H_1} \eta$ satisfies
$$
\E_{H_1}\eta \rightarrow 1, \E_{H_0} \eta \leq \alpha \mbox{ as } (n,p+q) \rightarrow \infty, \text{ for any fixed $\alpha > 0$}
$$
where one also needs to specify the relative rate at which $n,p,q$ can increase. In the following section, we will show that when $(p+q)/n \leq \kappa <\infty$, if $n$ does not grow faster than (a constant factor times) $\sqrt{pq}/\|\Sigma_{XY}\|_F^2$  then it cannot possibly be consistent.

\section{Main Theorem}\label{sec:thm}

Define the following set of covariance matrices, encoding a set of Gaussians:
$$\Theta(b) = \left\{\Sigma : \|\Sigma_{XY}\|_F = b\frac{\sqrt[4]{pq}}{\sqrt{2n}}\right\}.$$

\begin{theorem}
Let $0 < \alpha < \beta < 1$.  Suppose that as $n \rightarrow \infty, (p + q) \rightarrow \infty$ and that $(p + q) / n \leq \kappa$ for some constant $\kappa < \infty$ and all $n$.  Then there exists a constant $b = b(\kappa, \beta - \alpha) < 1$ such that for any test $\eta \in [\eta]_{n,p,q,\alpha}$  for testing $H_0: \Sigma_{XY} = 0$  vs  $H_1: \Sigma_{XY} \in \Theta(b)$,
$$
\limsup_{n\rightarrow\infty} \inf_{\Sigma\in\Theta(b)} \E_{H_1} \eta < \beta
$$
\end{theorem}

\begin{remark}
This theorem immediately implies the same bound holds for general linear independence testing (no algorithm can guarantee consistency with fewer samples, since that algorithm must definitely also work on our Gaussians, contradicting our lower bound).
\end{remark}

As mentioned before, one way to interpret the above theorem, is that if $n$ \textit{equals} a small constant times $\sqrt{pq}/\|\Sigma_{XY}\|_F^2$, then any test's minimax power must be bounded away from 1 (and if we get to control the constant, we can force the minimax power to be as close to $\alpha$ as we like). For consistency of any algorithm to be possible, $n$ needs to grow \textit{larger} than the above rate. We now turn to the proof of this theorem. For conciseness, we use $Z$ to refer to $(X,Y)$ pairs. We broadly follow the proof strategy of \cite{cai2013optimal}.

\begin{proof}
We define some ``least favorable'' subset of $\Theta(b)$ as follows:
$$
\Theta^*(b) = \left\{\Sigma_{uv} = [I_{(p+q) \times (p+q)}] + \frac{b}{\sqrt{2 n}\sqrt[4]{pq}}(uv' + vu')\right\}
$$
where $u = (\{\pm1\}^p,\{0\}^q)$, $v = (\{0\}^p,\{\pm1\}^q)$, making the size of the set equal to $2^{p+q}$. Note that the off-diagonal values of $\Sigma_{XX}, \Sigma_{YY}$ are $0$ and every element of the cross-covariance $\Sigma_{XY}$ is $\pm \frac{b}{\sqrt{2 n}\sqrt[4]{pq}}$. It is easy to verify that $\Theta^*(b) \subset \Theta(b)$.
\footnote{For $b=b(\kappa) < \frac1{\sqrt{\kappa}} \leq \sqrt{\frac{n}{p+q}}$, since $\sqrt{pq}\leq \frac{p+q}{2}$, we have $\|\Sigma_{XY}\|_F < \sqrt{\frac{n}{p+q}} \frac{\sqrt{p+q}}{\sqrt{4n}} = \frac1{2}$. Hence, writing $\Sigma$ as $ I + A$ for perturbation $A$, we know $\|A\|_{op}  = \|\Sigma - I\|_{op} \leq \|\Sigma - I\|_F < 1$. Since the eigenvalues of $\Sigma$ cannot be more than $\|A\|_{op} < 1$ away from those of  $I$, $\Sigma$ is a symmetric positive definite covariance matrix.}.

 Let $P_0$ be the probability measure when $Z_1,\dots,Z_n \sim \N_{p+q}(0,I)$, i.e. $\Sigma_{XY}=0$.  Let $P_{uv}$ be the probability measure when $Z_1,\dots,Z_n \sim \N_{p+q}(0,\Sigma_{uv})$ for some $\Sigma_{uv} \in \Theta^*(b)$. For any test, we require $P_0(\eta \text{ rejects } H_0) = \E_0\eta \leq \alpha$, and its power is $P_{uv}(\eta \text{ rejects } H_0) = \E_{uv}\eta$.  Let $P_1 = \frac{1}{2^{p+q}} \sum_{u,v}P_{uv}$ be the average probability measure of the possible $P_{uv}$'s, and denote expectation w.r.t $P_1$ as $\E_1$.  Then for any test $\eta$
\begin{eqnarray*}
\inf_{\Sigma\in\Theta(b)} \E_{H_1} \eta 
&\leq & \inf_{\Sigma_{uv} \in \Theta^*(b)} \E_{uv}\eta \\
& \leq & \frac{1}{2^{p+q}}\sum_{u,v} \E_{uv} \eta  \ =  \ \E_1 \eta \\
& \leq & \alpha + \E_1 \eta - \E_0 \eta \\
& \leq & \alpha + \sup_\eta | \E_1 \eta - \E_0 \eta | \\
& = &  \alpha + \frac{1}{2}\|P_1 - P_0\|_1
\end{eqnarray*}
The last equality follows because $ \sup_\eta | \E_1 \eta - \E_0 \eta |$ is precisely the total variation distance.
To control the rightmost side, we bound the $L_1$ distance by the chi-square divergence as
$$
\|P_1-P_0\|_1^2\leq \E_0 \left|\frac{dP_1}{dP_0}-1\right|^2 = \E_0\left|\frac{dP_1}{dP_0}\right|^2 - 1 = \int\frac{f_1^2}{f_0} - 1
$$

To prove the theorem, we will now perform detailed calculations to show that for an appropriate $b = b(\kappa,\beta-\alpha)$, we have
$$\int\frac{f_1^2}{f_0} - 1 \leq 4(\beta - \alpha)^2.$$


Noting that $\text{diag}(\Sigma_{uv}) = (1,\dots,1)$ and denoting $a=\frac{b}{\sqrt{2n}\sqrt[4]{pq}}$, the Sherman-Morrison formula gives us
$$
\Sigma_{uv}^{-1} = I - \frac{a(vu' + uv' - a(pvv' + quu'))}{1-pqa^2},
$$
and combining the Schur complement formula with Sylvester's determinant theorem, we see
$$
\text{det}(\Sigma_{uv}) = 1 - pqa^2.
$$
So, the density functions are as follows:
$$
f_0(z_1,\dots,z_n) = \frac{1}{(2\pi)^{n(p+q)/2}}\exp\left\{-\frac{1}{2}\sum_{i=1}^nz_i'z_i\right\}
$$
\begin{align*}
f_1(z_1,\dots,z_n) &= \frac{1}{2^{p+q}}\sum_{u,v}\frac{1}{(2\pi)^{n(p+q)/2}(\text{det}(\Sigma_{uv}))^{n/2}}\exp\left\{-\frac{1}{2}\sum_{i=1}^nz_i'\Sigma_{uv}^{-1}z_i\right\}\\
&=\frac{1}{(2\pi)^{n(p+q)/2}}\exp\left\{-\frac{1}{2}\sum_{i=1}^nz_i'z_i\right\}\frac{1}{(1-pqa^2)^{n/2}}\\
&\times \frac{1}{2^{p+q}}\sum_{u,v}\exp\left\{\frac{a}{2(1-pqa^2)}\sum_{i=1}^n 2(v'z_i)(u'z_i) - qa(u'z_i)^2 - pa(v'z_i)^2\right\}
\end{align*}
Therefore, substituting $t = \frac{a}{(1-pqa^2)} $, and $T(u,v,z) =  2(v'z)(u'z) -qa(u'z)^2 - pa(v'z)^2$
\begin{align*}
\frac{f_1^2}{f_0} &= \frac{1}{(1-pqa^2)^n}\frac{\exp\left\{-\frac{1}{2}\sum_{i=1}^nz_i'z_i\right\}}{(2\pi)^{n(p+q)/2}}
\frac{1}{2^{2(p+q)}}\biggl(\sum_{u,v}\exp\biggl[\frac{t}{2}\sum_{i=1}^n T(u,v,z_i) \biggr]\biggr)^2
\end{align*}

We can write the squared summation as a product by denoting $g \in (\{\pm1\}^p,\{0\}^q)$ and $h \in (\{0\}^p,\{\pm1\}^q)$ independent from $(u,v)$:
\begin{align*}
\frac{f_1^2}{f_0} &= \frac{1}{(1-pqa^2)^n}\frac{\exp\left\{-\frac{1}{2}\sum_{i=1}^nz_i'z_i\right\}}{(2\pi)^{n(p+q)/2}}
\frac{1}{2^{2(p+q)}}\biggl(\sum_{u,v}\sum_{g,h}\exp\biggl\{\frac{t}{2}\sum_{i=1}^n\biggl( T(u,v,z_i) + T(g,h,z_i) \biggr)\biggr\}\biggr)
\end{align*}
Hence, integrating w.r.t. to $z_1^n = (z_1,...,z_n)$, switching integrals and sums, we see that
\begin{align*}
\int \frac{f_1^2}{f_0} dz_1^n &= \frac{1}{(1-pqa^2)^n} \frac{1}{2^{2(p+q)}} \sum_{u,v,g,h} \int \frac{\exp\left\{-\frac{1}{2}\sum_{i=1}^nz_i'z_i\right\}}{(2\pi)^{(p+q)/2}} \exp\biggl[\frac{t}{2}\sum_{i=1}^n\biggl( T(u,v,z_i) + T(g,h,z_i) \biggr)\biggr] dz_1^n \\
&= \frac{1}{(1-pqa^2)^n} \frac{1}{2^{2(p+q)}} \sum_{u,v,g,h} \left[ \int \frac{\exp\left\{-\frac{1}{2}z'z\right\}}{(2\pi)^{(p+q)/2}} \exp\biggl\{\frac{t}{2}\biggl( T(u,v,z) + T(g,h,z) \biggr)\biggr\} dz \right]^n
\end{align*}

The last summand is interpretable as the $n$th power of the MGF of $ (T(u,v,z) + T(g,h,z)) /2 $ when $z$ is standard Gaussian. Denoting $\chi = (\chi_1,\chi_2,\chi_3,\chi_4)=(u'z,v'z,g'z,h'z)$, it is easy to algebraically verify that $T(u,v,z) + T(g,h,z ) = \chi^T A \chi $, where
\begin{align*}
A &=
\left[ \begin{array}{cccc}
-qa&1&0&0\\
1&-pa&0&0\\
0&0&-qa&1\\
0&0&1&-pa\\
\end{array} \right].
\end{align*}

Hence, we must now calculate the MGF of $\chi^T A \chi/2$ where, since $z$ is standard Gaussian,
\begin{align}
\chi &= 
\left[ \begin{array}{c}
\chi_1\\
\chi_2\\
\chi_3\\
\chi_4\\
\end{array} \right] \sim 
\N_4\left(\left[\begin{array}{c}
0\\
0\\
0\\
0
\end{array} \right],
\left[ \begin{array}{cccc}
p&0&u'g&0\\
0&q&0&v'h\\
u'g&0&p&0\\
0&v'h&0&q\\
\end{array} \right]\right).
\end{align}

 \cite[pg 29]{mathai1992quadratic} shows how to decompose arbitrary quadratic forms of dependent Gaussians into sums of potentially noncentral, \textit{but independent}, chi-squared random variables. Denoting $\Sigma$ as the aforementioned covariance of $\chi$, we have $\chi^TA\chi = \gamma_{00} \Psi_{00} + \gamma_{01} \Psi_{01} + \gamma_{10} \Psi_{10} + \gamma_{11} \Psi_{11}$ for iid $\Psi_{ij} \sim \chi^2_1$ where the $\gamma_{ij}$ are the eigenvalues of $\Sigma^{^1/_2}A\Sigma^{^1/_2}$,  written explicitly as:
\begin{eqnarray*}
\gamma_{ij} &=& \frac{1}{2}\biggl(-2apq + (-1)^i aq(u'g) + (-1)^i ap (v'h)  - (-1)^j \sqrt{R}\biggr)\\
\text{where ~} \ R &=& 4pq- (-1)^i 4q(u'g) + a^2q^2(u'g)^2 - (-1)^i 4p(v'h) \\
&& +4(u'g) (v'h)-2a^2pq(u'g) (v'h)+a^2p^2(v'h)^2
\end{eqnarray*}

Hence, resubstituting, we see that
\begin{align}
\int \frac{f_1^2}{f_0}  &= \frac{1}{(1-pqa^2)^n}\frac{1}{2^{2(p+q)}} \sum_{u,v,g,h} \mathbb{E}^n\left [ \exp\left\{\frac{t}{2}\gamma_{00}Y_{00}+\frac{t}{2}\gamma_{01}Y_{01}+ \frac{t}{2}\gamma_{10} Y_{10} +\frac{t}{2}\gamma_{11}Y_{11}\right\} \right] \nonumber \\
&= \frac{1}{(1-pqa^2)^n}\frac{1}{2^{2(p+q)}} \sum_{u,v,g,h} (1-t\gamma_{00})^{-\frac{n}{2}} (1-t\gamma_{01})^{-\frac{n}{2}} (1-t\gamma_{10})^{-\frac{n}{2}} (1-t\gamma_{11})^{-\frac{n}{2}}\\
&=\frac{1}{(1-pqa^2)^n} \E_{u,v,g,h} \left[ \left[(1-t\gamma_{00})(1-t\gamma_{01})(1-t\gamma_{10})(1-t\gamma_{11})\right]^{-\frac{n}{2}} \right]
\end{align}
where we used the chi-squared MGF $\E[\exp(t\gamma_{ij} Y_{ij}/2)] = (1-t\gamma_{ij})^{-1/2}$ which only holds if $t\gamma_{ij}/2 < 1/2$. We verify in Prop.\ref{prop:MGF} that this indeed holds for $b = b(\kappa) < \frac1{2\sqrt{\kappa}}$.

Note that $u'g$ is distributed identically to $u'\mathbf{1}$ when $u,g$ are i.i.d. Rademacher random variables, and likewise for $v'h$.  So we can safely replace all such instances with the scalars $U = u'\mathbf{1}$ and $V = v'\mathbf{1}$ respectively.
By substitution for $\gamma_{ij}$, the above simplifies to
\begin{align*}
\int \frac{f_1^2}{f_0} 
&=\frac{1}{(1-pqa^2)^n} \E_{u,v,g,h}\left[\left[\left(\frac{-1 + a^2(u'g)(v'h)}{-1 + a^2pq}\right)^{2}\right]^\frac{-n}{2} \right]\\
&=\E_{u,v,g,h}\left[(1-a^2(u'g)(v'h))^{-n}\right] =\E_{U,V}\left[(1-a^2UV)^{-n}\right]\\
&= \E_{U,V} \left[ \left[(1-a^2UV)^\frac{-1}{a^2UV}\right]^{na^2UV} \right]\\
&\leq \mathbb{E}_{U,V}\left[\exp\{na^2UV\log{4}\}\right]
\end{align*}
which is true because\footnote{Note that $|a^2UV| = |\frac{b^2}{2n \sqrt{pq}} UV| \leq |b^2 \frac{\sqrt{pq}}{2n}| < |\frac{b^2\kappa}{4}| < \frac{1}{2}$ if $b = b(\kappa) < \frac1{\sqrt \kappa}$, since $|U|\leq p,|V|\leq q$.} $x=a^2UV \leq \frac{1}{2}$ if $b = b(\kappa) < \frac1{\sqrt \kappa}$ and also\footnote{To show this, we can rewrite as $-\frac{1}{x}\log (1-x) \leq 2\log2$.  The only discontinuity on $x \leq \frac{1}{2}$ is at $x = 0$.  By L'H\^{o}pital's rule, 
$\lim_{x\rightarrow 0} -\frac{1}{x}\log (1-x) = \lim_{x\rightarrow 0} \frac{1}{1-x} = 1$, establishing continuity at $0$.
The derivative of $-\frac{1}{x}\log (1-x)$ 
is $\left(\frac{x}{1-x} + \log{(1-x)}\right)/x^2$. Since $\log{(1-x)} > -\frac{x}{1-x}$ for $x<1$, due to the Taylor expansion of $(1-x)\log{(1-x)}$, the aforementioned derivative is positive on $x \leq \frac{1}{2}$.  Therefore, $-\frac{1}{x}\log (1-x)$ is strictly increasing on $x \leq \frac{1}{2}$ and achieves value $2\log2$ at $x = \frac{1}{2}$.  Hence we can conclude that, as $-\frac{1}{x}\log (1-x) \leq 2\log 2$ on the desired interval, then $(1-x)^\frac{-1}{x} \leq 4$ on $x \leq \frac{1}{2}$.}
$(1-x)^\frac{-1}{x} \leq 4$ for $x \leq \frac{1}{2}$.

 Recalling that $a = \frac{b}{\sqrt{2n}\sqrt[4]{pq}}$
\begin{align}
\int \frac{f_1^2}{f_0}
&\leq \mathbb{E}_{U,V}\left[\exp\left\{n\frac{b^2\sqrt{pq}}{2npq}UV\log{4}\right\}\right]
\leq \mathbb{E}_{U,V}\left[\exp\left\{\frac{b^2}{\sqrt{pq}}|UV|\log{2}\right\}\right] \nonumber\\
&=\int^{\infty}_0 P\left(\exp\left\{\frac{b^2}{\sqrt{pq}}|UV|\log{2}\right\} \geq \mu\right)d\mu \nonumber\\
&\leq1+\int^{\infty}_1P\left(\exp\left\{\frac{b^2}{\sqrt{pq}}|UV|\log{2}\right\} \geq \mu\right)d\mu \nonumber\\
&=1+\int^{\infty}_1P\left(|UV| \geq \frac{\log{\mu}}{\log{2}}\frac{\sqrt{pq}}{b^2}\right)d\mu \nonumber\\
&\leq 1+\int^{\infty}_12e^{\frac{-p\log{\mu}}{pb^22\log{2}}}+2e^{\frac{-q\log{\mu}}{qb^22\log{2}}}d\mu\label{eq:hoeff-sub} \\
&=1+\int^{\infty}_14\exp\left\{-\frac{\log{\mu}}{\log{4}}\frac{1}{b^2}\right\}d\mu =1+\int^{\infty}_14\mu^{\frac{-1}{b^2\log{4}}}d\mu \nonumber\\
&=1+4\frac{b^2\log{4}}{1-b^2\log{4}} \nonumber
\end{align}
The last equality holds when $b< 1/\sqrt{\log 4}$ and Eq.\eqref{eq:hoeff-sub} follows by Hoeffding's inequality\footnote{Setting $\sqrt{C_1} = \sqrt{\frac{p\log{\mu}}{b^2\log{2}}}$ and $\sqrt{C_2} = \sqrt{\frac{q\log{\mu}}{b^2\log{2}}}$ so that $\sqrt{C_1C_2} = \frac{\log{\mu}}{\log{2}}\frac{\sqrt{pq}}{b^2}$, Hoeffding's inequality implies that $P(|U| \leq \sqrt{C_1}) \geq 1-2e^{\frac{-C_1}{2p}}$ and
$P(|V| \leq \sqrt{C_2}) \geq 1-2e^{\frac{-C_2}{2q}} $, and because $U,V$ are independent,
\begin{align}
P(|UV| \geq \sqrt{C_1C_2}) &\leq 1 - P\left(|U| \leq \sqrt{C_1} \;\ \text{and} \;|V| \leq \sqrt{C_2} \right) \nonumber\\
&\leq1 - (1-2e^{\frac{-C_1}{2p}})(1-2e^{\frac{-C_2}{2q}}) 
\leq 2e^{\frac{-C_1}{2p}}+2e^{\frac{-C_2}{2q}} \nonumber
\end{align}
}. Finally, observe that choosing any  $b \leq \frac{(\beta - \alpha)}{\sqrt{\log 4}{(1 + \beta - \alpha)}}$ guarantees
$$\int\frac{f_1^2}{f_0} - 1 \leq  4(\beta-\alpha)^2.$$

This concludes the proof of our main theorem, except for Prop. \ref{prop:MGF} which follows.
\end{proof}

\begin{proposition}\label{prop:MGF}
For $t$ and $\gamma_{ij}$ as defined, we have $t \gamma_{ij}/2 < 1/2$ whenever $b < 1/2\sqrt{\kappa}$.
\end{proposition}
\begin{proof}
Note that $\max_{u,g}|u'g|=p$, $\max_{v,h}|v'h|=q$.
By taking every term in each $\gamma_{ij}$ to be maximal and positive, we see that $\gamma_{ij} \leq \frac{1}{2}\left(4apq+\sqrt{16pq+4a^2p^2q^2}\right)$.
Since $2\sqrt{pq} \leq (p+q)$, recall that $a= \frac{b}{\sqrt{2n}\sqrt[4]{pq}} < \frac1{2} \sqrt{\frac{n}{p+q}} \frac1{\sqrt{2n}\sqrt[4]{pq}}  < \frac1{4\sqrt{pq}}$ and $t = \frac{a}{(1-pqa^2)}$ 
\begin{align}
\frac{t}{2}\gamma_{ij} &\leq \frac{a\sqrt{pq}}{1-pqa^2}\frac{1}{4}\left(4a\sqrt{pq}+\sqrt{16+4a^2pq}\right) \nonumber \\
&\leq \frac{1/4}{1 -1/16} \frac1{4} \left( 4 \cdot \frac1{4} + \sqrt{16 + 4 \cdot 1/16} \right) < 1/2 \nonumber
\end{align}
\end{proof}



\section{Tightness of the lower bound when $q=1$}\label{sec:upper}

We provide two pieces of evidence that the lower bound is morally tight, by examining the special case when $q=1$, and showing that existing tests achieve the aforementioned rate in special cases. Tightness in the general case is still open, but it would not be surprising.

\subsection{Linear Regression}

Consider the classic problem of linear regression under a Gaussian noise assumption. Given $(X_1,Y_1),...,(X_n,Y_n) \in \R^{p+1}$ from underlying linear model
$$
Y_i = X_i^T  \beta + e_i
$$
where $e_i \sim \N(0,\sigma^2)$ and $\beta \in \R^p$ is the unknown set of regression coefficients. 
In other words, $Y_i$ is drawn from a univariate conditional Gaussian distribution $Y | (X=X_i)$ with density 
$
\N(X_i^T \beta, \sigma^2 ).
$ 
It is clear from this expression that $Y$ is independent of $X$ iff $\beta = 0$ whenever $\Sigma_X$ is invertible. Hence, one can perform independence testing by estimating $\hat \beta$ and testing whether it is far from zero or not. How does such an approach compare to a direct independence test that does not proceed through prediction?

Let us first translate this setting into a comparable form to the earlier independence testing setting. Assume that the marginal distribution of $X$ is also Gaussian $\N_p(0,\Sigma_X)$ for a positive definite, invertible $\Sigma_X$. Then the joint distribution of $(X,Y)$ has the density
$$
\N_{p+1} \left( 0 , \left[ \begin{matrix} \Sigma_X & \Sigma_X \beta \\ \beta^T \Sigma_X & \sigma^2 + \beta^T \Sigma_X \beta \end{matrix} \right] \right)
$$ 
as can be  verified by the Schur-complement formulae. We can consider the cross-covariance $\Sigma_{XY}$ to just be the vector $\Sigma_X \beta$ and hence $\Sigma_{XY} = 0$ iff $\beta = 0$.

When specialized to $\Sigma_X = I$, $\sigma^2=1$, i.e. orthonormal design setting, Theorem 3.2 in \cite{wang2013generalized} suggests that their generalized F-test statistic has non-trivial power whenever $p/n \in (0,1)$, and $n$ is at least as large as $\sqrt{p}/\|\beta\|_2^2$. This matches our conditions when $q=1$.

\subsection{Two Sample Testing}


Two sample testing is another fundamental decision-theoretic problem, closely related to independence testing as we shall see below.

Consider  getting data $(X,Y)$ from the following generative model: $W\sim \text{Ber}(1/2)$, $X |W \sim W  {\cal N}_p(\mu_{1},I) + (1-W) {\cal N}(\mu_{2},I_p)$ and $Y = 2W-1$. Suppose we want to test whether $X$ is linearly independent of $Y$. In this case, 
$q=1$ and $\Sigma_{XY} = \E[(X - (\mu_{1} + \mu_{2})/2)(Y)] = (\mu_{1}-\mu_{2})/2$.

The reader may notice on closer inspection that this just a two-sample testing problem in disguise, where we are testing if $\mu_{1}=\mu_{2}$ or not. Our theorem predicts that any test will have non-trivial power only when $n$ is larger than $\sqrt{p}/\|\mu_{1} - \mu_{2}\|^2$. This matches the lower and upper bounds given in \cite{RamRed15c}, Theorem 4 and Section 4.1, in the special case of $q=1$ and the covariances of $Z_1,Z_2$ both being the identity.

%

\section{Conclusion}

In this paper, we prove the first minimax bounds for linear independence testing (without sparsity assumptions). Interesting open problems include (a) finding matching upper bounds under general settings, presumably with a test statistic of the form $\|\hat \Sigma_{XY}\|_F^2$, (b) finding lower bounds in settings with sparsity, and (c) proving lower bounds for more general \textit{nonlinear} independence testing settings (a topic of recent interest using kernels of distance based methods).

\subsection*{Acknowledgments}

The authors acknowledge NSF grant IIS-1247658 and AFOSR YIP FA9550-14-1-0285.

\bibliography{thesis}
\bibliographystyle{unsrt}

\end{document}